\documentclass[a4,a4wide]{article}
\usepackage{aaai}

\usepackage{times}
\usepackage{helvet}
\usepackage{courier}
\usepackage{latexsym}
\usepackage{amsmath}
\usepackage{amssymb}
\usepackage{url}
\usepackage{ifpdf}
\usepackage{enumerate}
\usepackage{paralist}
\usepackage{alltt}
\usepackage{verbatim} 
\usepackage{color}
\usepackage[T1]{fontenc}

\def\qed{\hfill{}\ensuremath{\Box}}
\newtheorem{theorem}{Theorem}
\newtheorem{lemma}{Lemma}

\newtheorem{definition}{Definition}
\newtheorem{proof}{Proof}
\newtheorem{example}{Example}

\newenvironment{appproof}[2]%
   {\par\noindent{\bf Proof of #1}.\quad #2}%
   {\qed}
\newenvironment{appsketch}[2]%
   {\par\noindent{\bf Proof of #1} (\emph{Sketch\/}).\quad #2}%
   {\qed}

\newlength{\phantomlength}
\def\ba{\begin{array}}
\def\ea{\end{array}}

\def\bitem{\vspace{0ex}\begin{itemize}}
\def\eitem{\end{itemize}\vspace{0ex}}

\newcommand{\naf}{not~}

\newcommand{\utimes}[0]{\otimes}

\newcommand{\snot}[0]{\naf}
\newcommand{\st}[0]{s.t.\ }

\newcommand{\iec}[0]{i.e.,\ }
\newcommand{\egc}[0]{e.g.,\ }
\newcommand{\commadots}[0]{,\ldots ,}

\def\nota#1{\marginpar{#1}}
\long\def\CD#1{{\textcolor{black}{#1}}}

\long\def\JM#1{{\textcolor{blue}{#1}}}

\begin{document}
 \nocopyright
\pagenumbering{gobble}
\title{Generalizing Modular Logic Programs \thanks{The work of Jo\~ao Moura was supported by the grant SFRH/BD/69006/2010 from Funda\c{c}\~ao para a Ci\^encia e Tecnologia (FCT) from the Portuguese MEC - Minist\'erio do Ensino e da Ci\^encia. He would also like to thank Carlos Dam\'asio for his important contribution as well as the anonymous reviewers.}
}
\author{Jo\~ao Moura and Carlos Viegas Dam\'asio\\
CENTRIA - Centre for Artificial Intelligence\\
Universidade Nova de Lisboa, Portugal}

%% Use \authorrunning{Short Title} for an abbreviated version of
%% your contribution title if the original one is too long

% special commands
\newenvironment{myprogram}[2]{{\begin{figure}\label{#1}}\par}{\vspace{-.3in}{\end{figure}}}

\pagestyle{plain}
\setcounter{secnumdepth}{2} 
\setcounter{tocdepth}{2} 

\maketitle
%\newif\ifmakebbl\makebbltrue

% http://math.arizona.edu/~aprl/publications/mathclap/
\def\mathrlap{\mathpalette\mathrlapinternal}%
\def\mathrlapinternal#1#2{\rlap{$\mathsurround=0pt#1{#2}$}}%

\makeatletter
\newdimen\@mydimen%
\newdimen\@myHeightOfBar%
\settoheight{\@myHeightOfBar}{$|$}%
\newcommand{\SetScaleFactor}[1]{%
    \settoheight{\@mydimen}{#1}%
    \pgfmathsetmacro{\scaleFactor}{\@mydimen/\@myHeightOfBar}%
}%

\newcommand*{\Scale}[2][3]{\scalebox{#1}{\ensuremath{#2}}}%

\newcommand*{\nct}[3]{%
    \SetScaleFactor{\vphantom{\ensuremath{#1#2}}}% Compute scale to be applied
    #1%
\mathrel{\Scale[\scaleFactor]{|\mathrlap{\kern-0.48ex\sim}\hphantom{\kern-0.41ex\sim}}_#3}%
    #2%
}%

\begin{abstract}
Even though modularity has been studied extensively in conventional logic programming, there are few approaches on how to incorporate modularity into Answer Set Programming, a prominent rule-based declarative programming paradigm. A major approach is Oikarinnen and Janhunen's Gaifman-Shapiro-style architecture of program modules, which provides the composition of program modules. Their module theorem properly strengthens Lifschitz and Turner's splitting set theorem for normal logic programs. However, this approach is limited by module conditions that are imposed in order to ensure the compatibility of their module system with the stable model semantics, namely forcing output signatures of composing modules to be disjoint and disallowing positive cyclic dependencies between different modules. These conditions turn out to be too restrictive in practice and in this paper we discuss alternative ways of lift both restrictions independently, effectively solving the first, widening the applicability of this framework and the scope of the module theorem. 
\end{abstract}

\section{Introduction} 

Over the last few years, answer set programming (ASP) \cite{Eiter01computingpreferred,baral2003,570391,Marek99stablemodels,Niemela98logicprograms} emerged as one of the most important methods for declarative knowledge representation and reasoning. Despite its declarative nature, developing ASP programs resembles conventional programming: one often writes a series of gradually improving programs for solving a particular problem, \egc optimizing execution time and space. Until recently, ASP programs were considered as integral entities, which becomes problematic as programs become more complex, and their instances grow.
Even though modularity is extensively studied in logic programming, there are only a few approaches on how to incorporate it into ASP~\cite{GSPOPL89,OJtplp08,defk2009-iclp,journals/tplp/BabbL12} or other module-based constraint modeling frameworks~\cite{JarvisaloOJN:LPNMR09,Tasharrofi:2011:SAM:2050784.2050804}. The research on modular systems of logic program has followed two main-streams~\cite{DBLP:journals/jlp/BugliesiLM94}. One is programming in-the-large where compositional operators are defined in order to combine different modules, \egc~\cite{DBLP:conf/iclp/MancarellaP88,GSPOPL89,DBLP:conf/slp/OKeefe85a}. These operators allow combining programs algebraically, which does not require an extension of the theory of logic programs. The other direction is programming-in-the-small, \egc~\cite{Giordano199459,Miller86atheory}, aiming at enhancing logic programming with scoping and abstraction mechanisms available in other programming paradigms. This approach requires the introduction of new logical connectives in an extended logical language. The two mainstreams are thus quite divergent.

The approach of~\cite{OJtplp08} defines modules as structures specified by a program (knowledge rules) and by an interface defined by input and output atoms which for a single module are, naturally, disjoint. The authors also provide a module theorem capturing the compositionality of their module composition operator. However, two conditions are imposed: there cannot be positive cyclic dependencies between modules and there cannot be common output atoms in the modules being combined. Both introduce serious limitations, particularly in applications requiring integration of knowledge from different sources. The techniques used in~\cite{defk2009-iclp} for handling positive cycles among modules are shown not to be adaptable  for the setting of~\cite{OJtplp08}. 

In this paper we discuss two alternative solutions to the common outputs problem, generalizing the module theorem by allowing common output atoms in the interfaces of the modules being composed.  A use case for this requirement can be found in the following example.
\begin{example}\label{ex:usecase} Alice wants to buy a car, wanting it to be safe and not expensive; she preselected 3 cars, namely $c_1$, $c_2$ and $c_3$\@. Her friend Bob says that car $c_2$ is expensive, while Charlie says that car $c_3$ is expensive. Meanwhile, she consulted two car magazines reviewing all three cars. The first considered $c_1$ safe and  the second considered $c_1$ to be safe while saying that $c_3$ may be safe. %Furthermore, if a friend declares that a car is expensive, then she will consider it safe. 
Alice is very picky regarding safety, and so she seeks some kind of agreement between the reviews.

The described situation can be captured with five modules, one for Alice, other three for her friends, and another for each magazine. Alice should conclude that $c_1$ is safe since both magazines agree on this. Therefore, one would expect Alice to opt for car $c_1$ since it is not expensive, and it is reviewed as being safe. However, the current state-of-the-art does not provide any way of combining these modules since they share common output atoms.
$\hfill \blacksquare$\end{example}

In summary, the fundamental results of~\cite{OJtplp08} require a syntactic operation to combine modules -- basically corresponding to the union of programs --, and a compositional semantic operation joining the models of the modules. The module theorem states that the models of the combined modules can be obtained by applying the semantics of the natural join operation to the original models of the modules -- which is compositional. 

The authors show however that allowing common outputs destroys this property. There are two alternatives to pursue:

{\bf (1) Keep the syntactic operation:} use the union of programs to syntactically combine modules, plus some bookkeeping of the interface, and thus the semantic operation on models has to be changed;

{\bf (2) Keep the semantic operation:} the semantic operation is the natural join of models, and thus a new syntactic operation is required to guarantee compositionality.

Both  will be explored in this paper as they correspond to different and sensible ways of combining two sources of information, already identified in Example~\ref{ex:usecase}: the first alternative is necessary for Alice to determine if a car is expensive; the second alternative captures the way Alice determines whether a car is safe or not. Keeping the syntactic operation is shown to be impossible since models do not convey enough information to obtain compositionality. We present a solution to this problem based on a transformation that introduces the required extra information. The second solution is possible, and builds on the previous module transformation. 

This paper proceeds in Section 2 with an overview of the modular logic programming paradigm, identifying some of its shortcomings. In Section 3 we discuss alternative methods for lifting the restriction that disallows positive cyclic dependencies, and in Section 4 introduce two new forms of composing modules allowing common outputs, one keeping the original syntactic $union$ operator and the other keeping the original semantic model $join$ operator. %Section~\ref{sec:lifting} discusses how the methods for lifting both shortcommings can work together.  %Section~\ref{sec:compsemantics} presents a new notion of model that is fully compositional for the two forms of composition introduced in Section~\ref{section:generalising}. 
We finish with conclusions and a general discussion.

\section{Modularity in Answer Set Programming}
Modular aspects of Answer Set Programming have been clarified in recent years, with authors describing how and when two program parts (modules) can be composed~\cite{OJtplp08,defk2009-iclp,JarvisaloOJN:LPNMR09} under the stable model semantics. 
In this paper, we will make use of Oikarinen and Janhunen's  logic program modules defined in analogy to \cite{GSPOPL89} which we review after presenting the syntax of answer set programs.
\subsection{Answer set programming paradigm}

Logic programs in the answer set programming paradigm are formed by finite sets of rules $r$ having the following syntax:
$$ L_1\leftarrow L_{2}, \ldots, L_m, \snot  L_{m+1}, \ldots, \snot  L_n. \;\;(n \geq m \geq 0) {\bf(1)}$$ where each $L_i$ is a logical atom without the occurrence of function symbols -- arguments are either variables or constants of the logical alphabet. 

Considering a rule of the form {\bf (1)}, let $Head_P(r)=L_1$ be the literal in the head, $Body^+_P(r)=\{L_{2} \commadots L_m\}$ be the set with all positive literals in the body, $Body^-_P(r)=\{L_{m+1} \commadots L_n\}$ be the set containing all negative literals in the body, and
$Body_P(r)=\{L_{2} \commadots L_n\}$ be the set containing all literals in the body.
If a program is positive we will omit the superscript in $Body^+_P(r)$. Also, if the context is clear we will omit the subscript mentioning the program and write simply $Head(r)$ and $Body(r)$ as well as the argument mentioning the rule. 

The semantics of stable models is defined via the reduct operation~\cite{Gelfond88thestable}. Given an interpretation $M$ (a set of ground atoms), the reduct $P^M$ of a program $P$ with respect to $M$ is the program $$P^M = \{ Head(r) \leftarrow Body^+(r) \mid r \in P, Body^-(r) \cap M = \emptyset \}.$$ 
The interpretation $M$ is a stable model of $P$ iff $M = LM(P^M)$, where $LM(P^M)$ is the least model of program $P^M$\@. 

The syntax of logic programs has been extended with other constructs, namely weighted and choice rules~\cite{Niemela98logicprograms}. 
In particular, choice rules have the following  form: $$\{A_1\commadots A_n\} \leftarrow B_1, \ldots B_k, \naf C_1, \ldots, \naf C_m. (n\geq 1) {\bf(2)}$$  
As observed by~\cite{OJtplp08},
the heads of choice rules possessing multiple atoms can be freely split without affecting
their semantics. When splitting such rules into n different rules $\{a_i\} \leftarrow B_1, \ldots B_k, \naf C_1, \ldots, \naf C_m$ where $1 \leq i \leq n$, the only concern is the creation of $n$ copies of the rule body $B_1, \ldots B_k, \naf C_1, \ldots, \naf C_m.$ However, new atoms can be introduced to circumvent this. There is a translation of these choice rules to normal logic programs \cite{DBLP:journals/tplp/FerrarisL05}, which we assume is performed throughout this paper but that is omitted for readability. We deal only with ground programs and use variables as syntactic place-holders.

\subsection{Modular Logic Programming}
Modules, in the sense of~\cite{OJtplp08}, are essentially sets of rules with an input and output interface:

\begin{definition}[Program Module]
A logic program module $\mathcal{P}$ is a tuple $\langle R, I, O, H \rangle$ where:
	\begin{enumerate}
		\item $R$ is a finite set of rules;
		\item $I$, $O$, and $H$ are pairwise disjoint sets of input, output, and hidden atoms;
		\item $At(R) \subseteq At(\mathcal{P})$ defined by $At(\mathcal{P}) = I \cup O \cup H$; and
		\item $Head(R) \cap I =\emptyset $.
	\end{enumerate}
\end{definition}

The set of atoms in $At_v(\mathcal{P}) = I \cup O$ are considered to be \emph{visible} and hence accessible to other modules composed with $\mathcal{P}$ either to produce input for $\mathcal{P}$ or to make use of the output of $\mathcal{P}$. We use $At_i(\mathcal{P})=I$ and $At_o(\mathcal{P})=O$ to represent the input and output signatures of $\mathcal{P}$\@, respectively. The hidden atoms in $At_h(\mathcal{P}) = At(\mathcal{P})\backslash At_v(\mathcal{P})= H $ are used to formalize some auxiliary concepts of $\mathcal{P}$ which may not be sensible for other modules but may save space substantially. The condition $head(R) \not\in I$ ensures that a module may not interfere with its own input by defining input atoms of $I$ in terms of its rules. Thus, input atoms are only allowed to appear as conditions in rule bodies.

\begin{example}\label{example:Alice} The use case in Example \ref{ex:usecase} is encoded into the five modules shown here:
%in Figure~\ref{fig:usecase}.
%\begin{figure}
\[
\begin{array}{lll}
\mathcal{P}_A=<&\{ &buy(X) \leftarrow car(X), safe(X), \naf exp(X). \\
           && car(c_1). \quad car(c_2). \quad car(c_3). \}, \\
           &\{ &safe(c_1), safe(c_2), safe(c_3),\\
          && exp(c_1),exp(c_2), exp(c_3) \}, \\
            &\{ &buy(c_1), buy(c_2), buy(c_3)\},\\
          &\{ &car(c_1),car(c_2),car(c_3) \} >\\
\mathcal{P}_B=<&\{ &exp(c_2). \}, \{\}, \{ exp(c_2), exp(c_3) \}, \{ \} >\\
\mathcal{P}_C=<& \{ &exp(c_3). \}, \{\},\\
& \{ &exp(c_1), exp(c_2),  exp(c_3) \}, \{ \} >\\
\mathcal{P}_{mg_1}= <&\{ &safe(c_1). \}, \{ \},\\
& \{ &safe(c_1), safe(c_2),  safe(c_3) \}, \{ \} >\\
\mathcal{P}_{mg_2} =<&\{ &safe(X) \leftarrow car(X), airbag(X). \\
           & & car(c_1).\;  car(c_2). \; car(c_3). \;  airbag(c_1). \\
          &  \{ &airbag(c_3) \}. \;\;  \}, \\
           & \{ &\}, \{ safe(c_1), safe(c_2), safe(c_3) \}, \\
           & \{ &airbag(c_1), airbag(c_2), airbag(c_3),\\
         &&car(c_1),car(c_2),car(c_3)\} >\hfill \blacksquare
\end{array}
\]
%\caption{Representation of the use case}\label{fig:usecase}
%\end{figure}
\end{example}
In Example \ref{example:Alice}, module $\mathcal{P}_A$ encodes the rule used by Alice to decide if a car should be bought. The safe and expensive atoms are its inputs, and the buy atoms its outputs; it uses hidden atoms $car/1$ to represent the domain of variables. Modules $\mathcal{P}_B$, $\mathcal{P}_C$ and $\mathcal{P}_{mg_1}$ capture the factual information in Example~\ref{ex:usecase}. They have no input and no hidden atoms, but $Bob$ has only analyzed the price of cars $c_2$ and $c_3$\@. The ASP program module for the second magazine is more interesting\footnote{$car$ belongs to both hidden signatures of $\mathcal{P}_A$ and $\mathcal{P}_{mg_2}$ which is not allowed when composing these modules, but for clarity we omit a renaming of the $car/1$ predicate.}, and expresses the rule used to determine if a car is safe, namely that a car is safe if it has an airbag; it is known that car $c_1$ has an airbag, $c_2$ does not, and the choice rule states that car $c_3$ may or may  not have an airbag. \\

Next, the stable model semantics is generalized to cover modules by introducing a generalization of the Gelfond-Lifschitz's fixpoint definition. In addition to weekly default literals  (\iec $\naf $), also literals involving input atoms are used in the stability condition.  In ~\cite{OJtplp08}, the stable models of a module are defined as follows:

\begin{definition}[Stable Models of Modules]
An interpretation $M \subseteq At(\mathcal{P})$ is a stable model of an ASP program module $\mathcal{P} = \langle R, I, O,H \rangle$, if and only if $M = LM\left(R^M \cup \{a.|a\in M\cap I\}\right)$. The stable models of $\mathcal{P} $ are denoted by $AS(\mathcal{P})$\@.
\end{definition}

Intuitively, the stable models of a module are obtained from the stable models of the rules part, for each possible combination of the input atoms.

\begin{example}\label{example:3}Program modules $\mathcal{P}_B$, $\mathcal{P}_C$, and $\mathcal{P}_{mg_1}$ have each a single answer set $AS(\mathcal{P}_B)$ = $\{ \{ exp(c_2)\} \}$, $AS(\mathcal{P}_C)$ = $\{ \{ exp(c_3)\}\}$, and $AS(\mathcal{P}_{mg_1})$ = $\{ \{ safe(c_1) \} \}$. Module $\mathcal{P}_{mg_2}$ has two stable models, namely:
$\{ safe(c_1),$ $car(c_1),$ $car(c_2),$ $car(c_3),$ $airbag(c_1)\}$, and
$\{ safe(c_1),$ $safe(c_3),$ $car(c_1),$ $car(c_2),$ $car(c_3),$ $airbag(c_1),$ $airbag(c_3)\}$. 

Alice's ASP program module has $2^6=64$ models corresponding each to an input combination of safe and expensive atoms. Some of these models are:
\[
\begin{array}{lll}
\{ &buy(c_1), car(c_1), car(c_2), car(c_3), safe(c_1) &\}\\
\{ &buy(c_1), buy(c_3), car(c_1), car(c_2), car(c_3),\\ 
&safe(c_1), safe(c_3) &\} \\
\{ &buy(c_1), car(c_1), car(c_2), car(c_3),exp(c_3),\\ 
&safe(c_1), safe(c_3)& \}\hfill \blacksquare
\end{array}
\]
%$\hfill \blacksquare$
\end{example}

\subsection{Composing programs from models}

The composition of models is obtained from the union of program rules and by constructing the composed output set as the union of modules' output sets, thus removing from the input all the specified output atoms.~\cite{OJtplp08} define their first composition operator as follows: Given two modules $\mathcal{P}_1 = \langle R_1, I_1,O_1,H_1 \rangle$ and $\mathcal{P}_2 = \langle R_2, I_2,O_2,H_2 \rangle$, their composition $\mathcal{P}_1 \oplus \mathcal{P}_2$ is defined when their output signatures are disjoint, that is, $O_1 \cap O_2 = \emptyset$, and they respect each others hidden atoms, \iec $H_1 \cap At(\mathcal{P}_2) = \emptyset$ and $H_2 \cap At(\mathcal{P}_1) = \emptyset$. Then their composition is $$\mathcal{P}_1 \oplus \mathcal{P}_2 = \langle R_1 \cup R_2, (I_1 \backslash O_2) \cup (I_2 \backslash O_1), O_1 \cup O_2,H_1 \cup H_2 \rangle$$ 

However, the conditions given for $\oplus$ are not enough to guarantee compositionality in the case of answer sets and as such they define a restricted form:

\begin{definition}[Module Union Operator $\sqcup$]\label{definition:joinconditions} Given modules $\mathcal{P}_1, \mathcal{P}_2$, their union is $\mathcal{P}_1 \sqcup \mathcal{P}_2 = \mathcal{P}_1 \oplus \mathcal{P}_2$ whenever {\bf(i)} $\mathcal{P}_1 \oplus \mathcal{P}_2$ is defined and {\bf(ii)} $\mathcal{P}_1$ and $\mathcal{P}_2$ are mutually independent\footnote{There are no positive cyclic dependencies among rules in different modules, defined as loops through input and output signatures.}.
\end{definition}

Natural join ($\bowtie$) on visible atoms is used in~\cite{OJtplp08} to combine stable models of modules as follows: 
\begin{definition}[Join]
Given modules $\mathcal{P}_1$ and $\mathcal{P}_2$ and sets of interpretations $A_1 \subseteq 2^{At(\mathcal{P}_1)}$ and $A_2 \subseteq 2^{At(\mathcal{P}_2)}$, the natural join of $A_1$ and $A_2$ is: 
\[
\begin{array}{ll}
A_1 \bowtie A_2 = \{&M_1 \cup M_2 \mid M_1 \in A_1 , M_2 \in A_2 \text{ and }\\
&M_1 \cap At_v(\mathcal{P}_2) = M_2 \cap At_v(\mathcal{P}_ 1)\}
\end{array}
\]
\end{definition}

This leads to their main result, stating that:
\begin{theorem}[Module Theorem]
If $\mathcal{P}_1,\mathcal{P}_2$ are modules such that $\mathcal{P}_1 \sqcup \mathcal{P}_2$ is defined, then $$AS(\mathcal{P}_1 \sqcup \mathcal{P}_2) = AS(\mathcal{P}_1) \bowtie AS(\mathcal{P}_2)$$
\end{theorem}

Still according to~\cite{OJtplp08}, their module theorem also straightforwardly generalizes for a collection of modules because the module union operator $\sqcup$ is commutative, associative, and \CD{has the identity element $<\emptyset,\emptyset,\emptyset,\emptyset>$}.

\begin{example} \CD{Consider the composition $\mathcal{Q} = \left(\mathcal{P}_A  \sqcup \mathcal{P}_{mg_1} \right) \sqcup \mathcal{P}_{B} $\@. First, we have}
\[
\mathcal{P}_A  \sqcup \mathcal{P}_{mg_1}  = \left<
\begin{array}{ll}
\{ buy(X) \leftarrow car(X), safe(X),&\\ \quad\quad\quad\quad\quad\naf exp(X). &\\
\ \ car(c_1). ~ car(c_2).  ~car(c_3). ~ safe(c_1).\}, &\\
\{  exp(c_1), exp(c_2), exp(c_3) \}, \\
\{ buy(c_1), buy(c_2), buy(c_3),\\ 
\;\;safe(c_1), ~safe(c_2),~ safe(c_3) \},\\
\{ car(c_1),car(c_2),car(c_3) \}\\
\end{array}
\right>\hfill
\] 
\CD{It is immediate to see that the module theorem holds in this case. The visible atoms of $\mathcal{P}_A$ are $safe/1$, $exp/1$ and $buy/1$, and the visible atoms for $\mathcal{P}_{mg_1}$ are $\{safe(c_1), safe(c_2) \}$\@. The only model for $\mathcal{P}_{mg_1} = \{ safe(c_1) \}$ when naturally joined with the models of $\mathcal{P}_A$\@, results in eight possible models where $safe(c_1)$\@, $\naf safe(c_2)$\@, and $\naf safe(c_3)$ hold, and $exp/1$ vary. The final ASP program module $\mathcal{Q}$ is}
\[
\left<
\begin{array}{ll}
\{ buy(X) \leftarrow car(X), safe(X), \naf exp(X). &\\
\ \ car(c_1). ~ car(c_2). ~ car(c_3). ~ exp(c_2). ~ safe(c_1).\}, &\\
\{  exp(c_1) \}, \\
\{ buy(c_1), buy(c_2), buy(c_3), exp(c_2),\\ 
\;\;safe(c_1), safe(c_2), safe(c_3) \},\\
\{ car(c_1),car(c_2),car(c_3) \}\\
\end{array}
\right>
\] 
The stable models of $\mathcal{Q}$ are thus:
\[
\begin{array}{c}
\{ safe(c_1), exp(c_1), exp(c_2), car(c_1), car(c_2), car(c_3)  \}\hfill\\
\{ buy(c_1), safe(c_1), exp(c_2), car(c_1), car(c_2), car(c_3)  \}
\hfill \blacksquare
\end{array}
\]
\end{example}

\subsection{Visible and Modular Equivalence}
The notion of visible equivalence has been introduced in order to neglect hidden atoms when logic programs are compared on the basis of their models. The compositionality property from the module theorem enabled the authors to port this idea to the level of program modules--giving rise to modular equivalence of logic programs. 
%Visible and modular equivalence, denoted by the respective infix relation
%symbols $\equiv_v$ and $\equiv_m$, are formulated for answer set program modules as follows.

\begin{definition}
Given two logic program modules $\mathcal{P}$ and $\mathcal{Q}$, they are:\\
{\bf Visibly equivalent:} $\mathcal{P} \equiv_v \mathcal{Q}$ iff $At_v(\mathcal{P}) = At_v(\mathcal{Q})$ and there is a bijection $f : AS(\mathcal{P}) \rightarrow AS(\mathcal{Q})$ such that for all $M \in AS(\mathcal{P})$, $M \cap At_v(\mathcal{P}) = f(M) \cap At_v(\mathcal{Q})$.\\ 
{\bf Modularly equivalent:} $\mathcal{P} \equiv_m \mathcal{Q}$ iff $At_i(\mathcal{P}) = At_i(\mathcal{Q})$ and $\mathcal{P} \equiv_v Q$.
\end{definition}

So, two modules are visibly equivalent if there is a bijection among their stable models, and they coincide in their visible parts. If additionally, the two program modules have the same input and output atoms, then they are modularly equivalent. 

\subsection{Shortcomings}
The conditions imposed in these definitions bring about some shortcomings such as the fact that the output signatures of two modules must be disjoint which disallows many practical applications \egc we are not able to combine the results of program module $\mathcal{Q}$ with any of $\mathcal{P}_C$ or $\mathcal{P}_{mg_2}$\@, and thus it is impossible to obtain the combination of the five modules. Also because of this, the module union operator $\sqcup$ is not reflexive. By trivially waiving this condition, we immediately get problems with conflicting modules. 
The compatibility criterion for the operator $\bowtie$ also rules out the compositionality of mutually dependent modules, but allows positive loops inside modules or negative loops in general.

\begin{example}[Common Outputs]\label{example:incomplete} 
Given $ \mathcal{P}_B$ and  $\mathcal{P}_C$\@, which respectively have:

 $AS(\mathcal{P}_B$)=$\{\{exp(c_2)\}\}$ and $AS(\mathcal{P}_C)$=$\{\{ exp(c_3)\}\}$\@,\\
the single stable model of their union $AS(\mathcal{P}_B \sqcup \mathcal{P}_C)$ is: 
$$\{exp(c_2), exp(c_3)\}$$
However, the join of their stable models is $AS(\mathcal{P}_B) \bowtie AS(\mathcal{P}_C) = \emptyset$, invalidating the module theorem.%, since it holds that $AS(\mathcal{P}_B \sqcup \mathcal{P}_C) \neq AS(\mathcal{P}_B) \bowtie AS(\mathcal{P}_C)$.}
$\hfill \blacksquare$\end{example}

We illustrate next the issue with positive loops between modules.
\begin{example}[Cyclic Dependencies]\label{example:positiveloop} 
Take the following two program modules: 
\[
\begin{array}{c}
\mathcal{P}_1 = \langle \{airbag \leftarrow safe.\}, \{safe\},\{airbag\},\emptyset \rangle\\
\mathcal{P}_2= \langle \{safe \leftarrow airbag.\}, \{airbag\}, \{safe\}, \emptyset\rangle
\end{array}
\]
Their stable models are: 
$$AS(\mathcal{P}_1)=AS(\mathcal{P}_2)=\{\{\},\{airbag,safe\}\}$$ 
while the single stable model of the union $AS(\mathcal{P}_1 \sqcup \mathcal{P}_2$) is the empty model $\{\}$. Therefore $AS(\mathcal{P}_1 \sqcup \mathcal{P}_2) \neq$ $AS(\mathcal{P}_1) \bowtie AS(\mathcal{P}_2)$ = $\{\{\},\{airbag,safe\}\}$, thus also invalidating the module theorem.
$\hfill \blacksquare$\end{example}

%%%%%%%%%%%%%%%%%% LOOPS %%%%%%%%%%%%%%%%%%%%%%%%%%

\section{Positive Cyclic Dependencies Between Modules}\label{section:cycles}
To attain a generalized form of compositionality we need to be able to deal with the two restrictions identified previously, namely cyclic dependencies between modules. In the literature,~\cite{defk2009-iclp} presents a solution based on a model minimality property. It forces one to check for minimality on every comparable models of all program modules being composed. It is not applicable to our setting though, which can be seen in Example~\ref{example:counter} where logical constant $\bot$ represents value $false$.

\begin{example}[Problem with minimization]\label{example:counter}
Given modules $\mathcal{P}_1 = \langle \{a\leftarrow b. \; \bot \leftarrow not\; b.\},\{b\},\{a\},\{\} \rangle$ with one answer set $\{a,b\}$, and $\mathcal{P}_2 = \langle \{b\leftarrow a. \},\{a\},\{b\},\{\} \rangle$
with stable models $\{\}$ and $\{a,b\}$, their composition has no inputs and no intended stable models while their minimal join contains $\{a,b\}$.
$\hfill \blacksquare$\end{example}

Another possible solution requires the introduction of extra information in the models to be able to detect mutual positive dependencies. This need has been identified before~\cite{SlotaL12:kr} and is left for future work. 

\section{Generalizing Modularity in ASP by Allowing Common Outputs}\label{section:generalising}
After having identified the shortcomings in the literature, we proceed now to seeing how compositionality can be maintained while allowing modules to have common output atoms.
\begin{comment}redefining the way programs are composed from modules (composition operator $\oplus$), at the expense of having to redefine the join operator ($\bowtie$). We extend the module theorem, proving that we retain compositionality. 
\end{comment}
In this section we present two versions of compositions: \CD{(1) A relaxed composition operator ($\uplus$), aiming at maximizing information in the stable models of modules. Unfortunately, we show that this operation is not compositional. (2) A conservative composition operator ($\otimes$), aiming at maximizing compatibility of atoms in the stable models of modules. This version implies redefining the composition operator by resorting to a program transformation but uses the original join operator.}
\begin{comment}This version does not lend itself well to the definition of a converse decomposition operator due to the fact that the models of composed modules are made compatible in terms of their output signatures, which implies the intersection of their compatible models, and thus a decomposition operator would make us fall outside the conditions of the original theorem.
\end{comment}

\subsection{Extra module operations}

First, one requires fundamental operations for renaming atoms in the output signatures of modules with fresh ones:

\begin{definition}[Output renaming]\label{definition:primeTransformation}
Let $\mathcal{P}$ be the program module $\mathcal{P} = \langle R, I,O,H \rangle$\@, $o \in O$ and $o' \not\in At(\mathcal{P})$\@. The renamed output program module $\rho_{o' \leftarrow o}\left(\mathcal{P}\right)$ is the program module $\langle R' \cup \{\bot \leftarrow o', \naf o.\}, I \cup \{o\},\{o'\} \cup (O\setminus \{o\}),H \rangle$\@. The program part $R'$ is constructed by substituting the head of each rule $o \leftarrow Body$ in $R$ by $o' \leftarrow Body$\@. The heads of other rules remain unchanged, as well as the bodies of all rules.
\end{definition}

\begin{comment}
\begin{definition}[Output renaming]\label{definition:primeTransformation}
Let $\mathcal{P}_1$ and $\mathcal{P}_2$ be two \CD{program} modules $\mathcal{P}_1 = \langle R_1, I_1,O_1,H_1 \rangle$ and $\mathcal{P}_2 = \langle R_2, I_2,O_2,H_2 \rangle$ such that $O_1\cap O_2 \neq \emptyset$.
Let $l$ be \CD{an atom}, then we define its \CD{renaming} transformation as:
\[ l' = \left\{ \begin{array}{ll}
        $ if $l = q$ and $l\notin O_1\cap O_2$ then $l'=q;\\
        $ if $l = q$ and $l\in O_1\cap O_2$ then $l'=q'.
\end{array} \right. \]
\CD{where $q'$ is a new atom obtained by substituting the predicate name in $q$ with a newly introduced name}. \CD{The transformation can also be applied to sets of literals.} The set of literals $L'$ obtained from \CD{set of literals} $L$ by applying the renaming transformation to every literal in $L$:
$L'=\{l'\mid l \in L\}$\@.
\end{definition}
\end{comment}

\CD{Mark that, by making $o$ an input atom, the renaming operation can introduce extra stable models. However, the original stable models can be recovered by selecting the models where $o'$ has exactly the same truth-value of $o$\@. The constraint throws away models where $o'$ holds but not $o$\@. We will abuse notation and denote $\rho_{o_1' \leftarrow o_1}\left( \ldots \left(\rho_{o_n' \leftarrow o_n}(\mathcal{P})\right)\ldots \right)$ by $\rho_{\{o_1',\ldots,o_n'\} \leftarrow \{o_1,\ldots,o_n\}}\left(\mathcal{P}\right)$\@.}

\begin{example}[Renaming]
Recall the module representing Alice's conditions in Example~\ref{example:Alice}.
%\[
%\begin{array}{lcll}
%\mathcal{P}_A & = < &  \{ buy(X) \leftarrow car(X), safe(X), \naf exp(X). &\\
%         &  & \ \ car(c_1). \quad car(c_2). \quad car(c_3). \}, &\\
%         & &  \{ safe(c_1), safe(c_2), safe(c_3), \ exp(c_1), exp(c_2), exp(c_3) \}, \\
 %        & &  \{ buy(c_1), buy(c_2), buy(c_3)\}, \{ car(c_1),car(c_2),car(c_3) \} >\\
%\end{array}
%\]
Its renamed output program module $\rho_{o' \leftarrow o}\left(\mathcal{P}_A\right)$ is the program module:
\[
\begin{array}{ll}
\rho_{o' \leftarrow o}\left(\mathcal{P}_A\right) = < &  \{ buy'(X) \leftarrow car(X), safe(X),\\
	&\quad\quad\quad\quad\quad \naf exp(X). \\
        & \ \ car(c_1). \quad car(c_2). \quad car(c_3). \\
        & \bot \leftarrow buy(X)', \naf buy(X).\}, \\
        &  \{ buy(X), safe(c_1), safe(c_2), safe(c_3),\\ 
        & \;\; exp(c_1), exp(c_2), exp(c_3) \}, \\
        &  \{ buy'(c_1), buy'(c_2), buy'(c_3)\},\\
        &  \{ car(c_1),car(c_2),car(c_3) \} >\hfill\blacksquare
\end{array}
\]
\end{example}

Still before we dwell any deeper in this subject, we define operations useful to project or hide sets of atoms from a module.

\begin{definition}[Hiding and Projecting Atoms]
Let $\mathcal{P}= \langle R, I,O,H \rangle$ be a module and $S$ \CD{an arbitrary set of atoms}\@. If we want to \emph{Hide} (denoted as $\backslash$) $S$ from program module $\mathcal{P}$, we use $\mathcal{P}\backslash S= \langle R \cup \{ \{i\}. \mid i \in I \cap S \}, I \backslash S,O\backslash S, H \cup ( (I \cup O) \cap S) \rangle$.  Dually, we can \emph{Project} (denoted as $\mid$) over $S$ in the following way: $\mathcal{P}\mid_{S}= \langle R \cup \{ \{i\}. \mid i \in I \setminus S \}, I \cap S,O \cap S, H \cup ((I \cup O) \setminus S)\rangle$.
\end{definition}
Both operators \emph{Hide} and \emph{Project} do not change the stable models of the original program, i.e. $AS(\mathcal{P}) = AS(\mathcal{P}\backslash S) = AS(\mathcal{P}_{\mid S})$ but do change the set of visible atoms $At_v(\mathcal{P}\backslash S) = At_v(\mathcal{P})\backslash S$ and  $At_v(\mathcal{P} \mid S) = At_v(\mathcal{P}) \cap S$\@

\subsection{Relaxed Output Composition}
For the reasons presented before, we start by defining a generalized version of the composition operator, by removing the condition enforcing disjointness of the output signatures of the two modules being combined.
 
\begin{definition}[Relaxed Composition]
Given two modules $\mathcal{P}_1 = \langle R_1, I_1,O_1,H_1 \rangle$ and $\mathcal{P}_2 = \langle R_2, I_2,O_2,H_2 \rangle$, their composition $\mathcal{P}_1 \uplus \mathcal{P}_2$ is defined when they respect each others hidden atoms, \iec $H_1 \cap At(\mathcal{P}_2) = \emptyset$ and $H_2 \cap At(\mathcal{P}_1) = \emptyset$. Then their composition is $\mathcal{P}_1 \uplus \mathcal{P}_2 = \langle R_1 \cup R_2, \CD{(I_1 \cup I_2) \backslash (O_1 \cup O_2)}, O_1 \cup O_2,H_1 \cup H_2 \rangle$. \end{definition}

 \begin{comment}
The following theorem comes almost as an intuitive consequence of the previous definitions.

\begin{theorem}{\label{theorem:relaxedextension}}
$\uplus$ is a conservative extension of $\oplus$.
\end{theorem}

\begin{proof}(sketch)[Theorem \ref{theorem:relaxedextension}]This proof goes along the lines of the Proof for Theorem \ref{theorem:extension}.
\qed\end{proof}
\end{comment}

Obviously, the following important properties hold for $\uplus$\@:
\begin{lemma} The relaxed composition operator is reflexive, associative, commutative and has the identity element $<\emptyset,\emptyset,\emptyset,\emptyset>$\@.
\end{lemma}
Having defined the way to deal with common outputs in the composition of modules, we \CD{would like} to redefine the operator $\bowtie$ for combining the stable models of these modules. However, this is shown here to be impossible.

\begin{lemma}\label{lem:notcompos} The operation $\uplus$ is not compositional, i.e. for any join operation $\bowtie'$, it is not always the case that $AS(\mathcal{P}_1 \uplus \mathcal{P}_2) = AS(\mathcal{P}_1) \bowtie' AS(\mathcal{P}_2)$\@.  
\end{lemma}

As we have motivated in the introduction, it is important to applications to be able to use $\uplus$ to combine program modules, and retain some form of compositionality. The following definition presents a construction that adds the required information in order to be able to combine program modules using the original natural join.

\begin{definition}[Transformed Relaxed Composition]\label{definition:relaxedTransformation}
Consider the program modules $\mathcal{P}_1 = \langle R_1, I_1,O_1,H_1 \rangle$ and $\mathcal{P}_2 = \langle R_2, I_2,O_2,H_2 \rangle$\@.  Let $O = O_1\cap O_2$\@, and define the sets of newly introduced atoms $O'$=$\{o' \mid o \in O\}$ and $O''$=$\{o'' \mid$ $o \in O\}$\@. Construct program module:
\[
\begin{array}{ll}
\mathcal{P}_{union} = < R_{union}, O' \cup O'', O, \emptyset>\text{ where:}\\
R_{union} = \{ o \leftarrow o'. \mid o' \in O'\} \cup \{ o \leftarrow o''. \mid o'' \in O''\}.
\end{array}
\]
The transformed relaxed composition is defined as the program module
\[
\begin{array}{ll}
(\mathcal{P}_1 \uplus^{RT} \mathcal{P}_2) = &[\rho_{O' \leftarrow O}(\mathcal{P}_1) \sqcup \rho_{O'' \leftarrow O}(\mathcal{P}_2) \sqcup \mathcal{P}_{union}]\; \setminus\\& [O' \cup O''] \hfill
\end{array}
\]
\end{definition}

Intuitively, we rename the common output atoms in the original modules, and introduce an extra program module that unites the contributions of each module by a pair of rules for each common atom $o \leftarrow o'$ and $o \leftarrow o''$\@. We then hide all the auxiliary atoms to obtain the original visible signature. If $O=\emptyset$ then $\mathcal{P}_{union}$ is empty, and all the other modules are not altered, falling back to the original definition.

\begin{theorem}\label{theorem:newmodule}  Let $\mathcal{P}_1$  and $\mathcal{P}_2$ be arbitrary program modules without positive dependencies among them. Then, modules joined with operators $\uplus$ and  $\uplus^{RT}$ are modularly equivalent:
$$\mathcal{P}_1 \uplus \mathcal{P}_2 \equiv_m \mathcal{P}_1 \uplus^{RT} \mathcal{P}_2.$$
\end{theorem}

The important remark is that according to the original module theorem we have:
$AS(\rho_{O' \leftarrow O}(\mathcal{P}_1)$ $\sqcup$
 $\rho_{O'' \leftarrow O}(\mathcal{P}_2)$ $\sqcup$ 
 $\mathcal{P}_{union})$
 $=$ 
 $AS(\rho_{O' \leftarrow O}(\mathcal{P}_1))$ 
 $\bowtie AS( \rho_{O'' \leftarrow O}(\mathcal{P}_2))$ 
 $\bowtie$ $AS(\mathcal{P}_{union})$.
Therefore, from a semantical point of view, users can always substitute module $\mathcal{P}_1 \uplus \mathcal{P}_2$ by $\mathcal{P}_1 \uplus^{RT} \mathcal{P}_2$\@, which has an extra cost since the models of the renamed program modules may increase. This is, however, essential to regain compositionality.

\begin{example} Considering program modules $\mathcal{Q}_1$ $= <$ $\{a.\quad$ $\bot \leftarrow a,b.\},$ $\emptyset,$ $\{a,b\},$ $\emptyset >$ and $\mathcal{Q}_2 = \left< \{b.\}, \emptyset, \{b\}, \emptyset \right>$,  we have:
\[
\begin{array}{llll}
\rho_{a',b' \leftarrow a,b}(\mathcal{P}_1) = &<& \{&a'. ~~ \bot \leftarrow a', \naf a.\\
&&&\bot \leftarrow b', \naf b.  \},\\ 
&& \{&a,b\}, \{a',b'\}, \emptyset \hfill>\\ 
\rho_{a'',b'' \leftarrow a,b}(\mathcal{P}_2) = &<& \{ &b''. ~~ \bot \leftarrow a'', \naf a. \\
&&&\bot \leftarrow b'', \naf b.\},\\ 
&&\{&a,b\}, \{a'',b''\}, \emptyset \hfill>\\
\mathcal{P}_{union} = &<& \{ &a \leftarrow a'. ~~ a \leftarrow a''.\\
&&&  b \leftarrow b'. ~~ b \leftarrow b''.\},\\ 
&&\{&a',a'',b',b''\}, \{a,b\}, \emptyset \hfill>\\
\rho_{a',b' \leftarrow a,b}(\mathcal{Q}_1) = &<& \{ &a'. ~~ \bot \leftarrow a, b.\\
&&& \bot \leftarrow a', \naf a. \\
&&&\bot \leftarrow b', \naf b.\},\\
&& \{&a,b\}, \{a',b'\}, \emptyset \hfill>\\ 
\rho_{a'',b'' \leftarrow a,b}(\mathcal{Q}_2) = &&&\rho_{a'',b'' \leftarrow a,b}(\mathcal{P}_2)\\
\mathcal{Q}_3 = \mathcal{P}_{union}\\
\end{array}
\]
The stable models of the first two modules are $\{\{a,a'\},\{a,b,a'\}\}$ and $\{\{b,b''\},\{a,b,b''\}\}$\@, respectively. Their join is $\{\{a,b,a',b''\}\}$ and the returned model belongs to $\mathcal{P}_{union}$ (and thus it is compatible), and corresponds to the only intended model $\{a,b\}$ of $\mathcal{P}_1 \uplus \mathcal{P}_2$\@. Note that the stable models of $\mathcal{P}_{union}$ are 16, corresponding to the models of propositional formula $(a \equiv a' \vee a'') \wedge (b \equiv b' \vee b'')$\@. Regarding, the transformed module $\rho_{a',b' \leftarrow a,b}(\mathcal{Q}_1)$ it discards the model $\{a,b,a'\}$, having stable models $\{\{a,a'\}\}$. But now the join is empty, as intended. 
\begin{comment}Finally, the combination of modules for Bob and Charlie, correspond almost exactly to the first situation we analysed, except that there is a non-common output atom with no fact for it, but it does not affect the previous discussion.
\end{comment}
 $\hfill \blacksquare$\end{example}

\subsection{Conservative Output Composition}
\begin{comment}The following composition method aims at maximizing compatibility of atoms in the stable models of their modules. We start by defining model compatibility for the current setting as models that have the same supported outputs:

\begin{definition}[Conservative Model Compatibility]
Let  $\mathcal{P}_1$ and $\mathcal{P}_2$ be two modules and $AS(\mathcal{P}_1)$, respectively $AS(\mathcal{P}_2)$ be their stable models.  Let now $M_1\in AS(\mathcal{P}_1)$ and $M_2\in AS(\mathcal{P}_2)$ be two models of the modules, they will be conservatively compatible models if for every $o \in O$, either $M_1$ contains $o$ and $o'$ and $M_2$ contains $o$ and $o''$, or neither one contains $o$, $o'$ or $o''$.
\end{definition}
\end{comment}

In order to preserve the original outer join operator, which is widely used in databases, for the form of composition we introduce next one must redefine the original composition operator ($\oplus$). We do that resorting to a program transformation \st the composition operator remains compositional with respect to the join operator ($\bowtie$). The transformation we present next consists of taking Definition \ref{definition:relaxedTransformation} and adding an extra module to guarantee that only compatible models (models that coincide on the visible part) are retained.

\begin{definition}[Conservative Composition]\label{definition:compatibilityTransformation}
Let $\mathcal{P}_1 = \langle R_1, I_1,O_1,H_1 \rangle$ and $\mathcal{P}_2 = \langle R_2, I_2,O_2,H_2 \rangle$ be modules such that their outputs are disjoint $O = O_1\cap O_2 \neq \emptyset$. 
Let $O'=\{o' \mid o \in O\}$ and $O''=\{o'' \mid o \in O\}$ be sets of newly introduced atoms. 

Construct program modules: 
\[
\begin{array}{ll}
\mathcal{P}_{union} =& < R_{union}, O' \cup O'', O, \emptyset>\text{ where:}\\
R_{union} =& \{ o \leftarrow o'. \mid o' \in O'\} \cup \{ o \leftarrow o''. \mid o'' \in O''\}.\\
\mathcal{P}_{filter} =& < \{\bot \leftarrow o', \snot o''.\; \bot \leftarrow \snot o', o''. \mid o \in O\},\\
& O' \cup O'', \emptyset, \emptyset>
\end{array}
\]

The conservative composition is defined as the program module: $\mathcal{P}_1 \utimes \mathcal{P}_2 = \lbrack( \rho_{O' \leftarrow O}(\mathcal{P}_1) \sqcup \rho_{O'' \leftarrow O}(\mathcal{P}_2) \sqcup \mathcal{P}_{union} \sqcup \mathcal{P}_{filter} \rbrack \setminus \left( O' \cup O''\right)$.
\end{definition}

Note here that each clause not containing atoms that belong to $O_1\cap O_2$ in $\mathcal{P}_1 \cup \mathcal{P}_2$ is included in $\mathcal{P}_1 \utimes \mathcal{P}_2$\@. So, if there are no common output atoms the original union based composition is obtained. Therefore, it is easy to see that this transformational semantics ($\utimes$) is a conservative extension to the existing one ($\oplus$).

\begin{theorem}[Conservative Module Theorem] \label{theorem:conservativemoduletheorem} 
If $\mathcal{P}_1, \mathcal{P}_2$ are modules such that $\mathcal{P}_1 \otimes \mathcal{P}_2$ is defined, then a model $M$ $\in$ $AS(\mathcal{P}_1$ $\otimes$ $\mathcal{P}_2)$ iff $M$ $\cap$ $(At(\mathcal{P}_1)$ $\cup$ $At(\mathcal{P}_2) )$  $\in$ $AS(\mathcal{P}_1)$ $\bowtie AS(\mathcal{P}_2)$\@.
\end{theorem}

The above theorem is very similar to the original Module Theorem of Oikarinnen and Janhunen apart from the extra renamed atoms required in  $\mathcal{P}_1 \otimes \mathcal{P}_2$ to obtain compositionality.

\begin{comment}
We prove Theorem \ref{theorem:conservativemoduletheorem} by reduction of its conditions and application to the conditions of the original Module Theorem. 

If $\mathcal{P}_1 \otimes \mathcal{P}_2$ is defined then let their transformed composition be $T = (\mathcal{P}_1 \otimes \mathcal{P}_2)^{T}$. Now let $R_1^{T} = T\backslash R(\mathcal{P}_2)$ (respectively $R_2^{T} = T\backslash R(\mathcal{P}_1)$) be the relaxed composition transformation of the rules in $\mathcal{P}_1$ (respectively $\mathcal{P}_2$). It is evident that $R_1^{T} \cap R_2^{T} \cap O_1 \cap O_2 = \emptyset$ meaning loosely that the output signatures of $R_1^{T}$ and $R_2^{T}$ are disjoint. Because of this, we are in the conditions of the Module Theorem and thus it is applicable to the result of the modified composition $\otimes$ iff the transformation did not introduce positive loops between $R_1^{T}$ and $R_2^{T}$: Rule (1) potentially produces only an annotation of an already existing rule, Rules (3) and (4) introduce integrity constraints 
while rule (2) introduces rules which contain in their heads only atoms that are not in the language of $R_1^{T}$ nor $R_2^{T}$. 
\qed
\end{comment}

\begin{example} Returning to the introductory example, we can conclude that $\mathcal{P}_{mg_1} \otimes \mathcal{P}_{mg_2}$ has only one answer set: $$\{safe(c_1),airbag(c_1),car(c_1), car(c_2), car(c_3)\}$$ since this is the only compatible model between $\mathcal{P}_{mg_1}$ and $\mathcal{P}_{mg_2}$\@. The stable models of 
$\rho(\mathcal{P}_{mg_1})$ and $\rho(\mathcal{P}_{mg_2})$\@, are collected in the table below where compatible models appear in the same row and $car(c_1), car(c_2), car(c_3)$ has been omitted from $AS(\rho(\mathcal{P}_{mg_2}))$. Atom $s$ (respectively $a$) stands for $safe$ (respectively $airbag$)\@.
\[
\begin{array}{|c|c|}\hline
\text{Answer sets of } \rho(\mathcal{P}_{mg_1}) & \text{Answer sets of } \rho(\mathcal{P}_{mg_2})\\\hline
\{s(c_1),s'(c_1)\} &  \{s(c_1),s''(c_1),a(c_1) \} \\\hline
\{s(c_1),s(c_2),s'(c_1)\} & \{s(c_1),s(c_2),s''(c_1),a(c_1) \} \\\hline
\{s(c_1),s(c_3),s'(c_1)\}  & \{s(c_1),s(c_3),s''(c_1),a(c_1)\}\\
& \{s(c_1),s(c_3),s''(c_1),\\ 
&s''(c_3),a(c_1),a(c_3)\}\\\hline
\{s(c_1),s(c_2),s(c_3),& \{s(c_1),s(c_2),s(c_3),\\
s'(c_1)\}& s''(c_1),a(c_1)\}\\
&\{s(c_1),s(c_2),s(c_3),s''(c_1),\\
&s''(c_3),a(c_1),a(c_3),c(c_1)\}
\\\hline
\end{array}
\]
The only compatible model retained after composing with $\mathcal{P}_{union}$ and $\mathcal{P}_{filter}$ is the combination of the stable models in the first row: $$\{s(c_1),s'(c_1), s''(c_1),a(c_1),c(c_1), c(c_2), c(c_3)\}.$$
Naturaly, this corresponds to the intended result if we ignore the $s'$ and $s''$ atoms.
$\hfill \blacksquare$\end{example}
We underline that models of composition $\mathcal{P}_1 \otimes \mathcal{P}_2$ will either contain all atoms $o$, $o'$\@, and $o''$ or none of them, and will only join compatible models from $\mathcal{P}_1$ having $\{o, o'\}$ with models in $\mathcal{P}_2$ having $\{o, o''\}$, or models without atoms in $\{o, o', o''\}$. 

%\subsection{Shortcomings revisited}\label{sec:lifting}
%The solutions for the shortcomings in the original framework by Oikarinen and Janhunen have been so far discussed independently. Here we discuss the way they interact, allowing one to lift both conditions at the same time.
\paragraph{Shortcomings Revisited}
The resulting models of composing modules using the transformation and renaming methods described so far in this Section 4 can be minimised a posteriori following the minimization method described in Section 3.

\subsection{Complexity}
Regarding complexity, checking the existence of $M \in P_1 \oplus P_2$ and $M \in P_1 \uplus^{RT} P_2$ is an NP-complete problem. It is immediate to define a decision algorithm belonging to $\Sigma^p_2$ that checks existence of a stable model of the module composition operators. This is strictly less than the results in the approach of~\cite{defk2009-iclp} where the existence decision problem for propositional theories is NEXP$^{\textrm{NP}}$-complete -- however their approach allows disjunctive rules.

\section{Conclusions and Future Work}
We redefined the necessary operators in order to relax the conditions for combining modules with common atoms in their output signatures. Two alternative solutions are presented, both allowing us to retain compositionality while dealing with a more general setting than before.
\cite{defk2009-iclp} provide an embedding of the original composition operator of Oikarinen and Janhunen into their approach. Since our constructions rely on a transformational approach using operator $\sqcup$ of Oikarinen and Janhunen, by composing both translations, an embedding into \cite{defk2009-iclp} is immediately obtained. It remains to be checked whether the same translation can be used in the presence of positive cycles. \cite{Tasharrofi:2011:SAM:2050784.2050804} take ~\cite{DBLP:journals/jair/JanhunenOTW09} and extend it with an algebra which includes a new operation of feedback (loop) over modules. They have shown that the loop operation adds significant expressive power -- modules can can express all (and only) problems in NP. The other issues remain unsolved though. 

%\JM{ Still, the approach in~\cite{defk2009-iclp} is more general than the one we take here but at the expense of complexity, which is in the worst case at least exponentially higher. This is partly because of the call by value semantics they introduced, which force some name resolution to be carried out at the semantic level. }

The module theorem has been extended to the general theory of stable models~\cite{journals/tplp/BabbL12}, being applied to non-ground logic programs containing choice rules, the count aggregate, and nested expressions. It is based on the new findings about the relationship between the module theorem and the splitting theorem. It retains the composition condition of disjoint outputs and still forbids positive dependencies between modules. As for disjunctive versions,~\cite{DBLP:journals/jair/JanhunenOTW09} introduced a formal framework for modular programming in the context of DLPs under stable-model semantics. This is based on the notion of DLP-functions, which resort to appropriate input/output interfacing. Similar module concepts have already been studied for the cases of normal logic programs and ASPs and even propositional theories, but the special characteristics of disjunctive rules are properly taken into account in the syntactic and semantic definitions of DLP functions presented therein. In \cite{Gebser:2011:RAS:2010192.2010201}, MLP is used as a basis for Reactive Answer Set Programming, aiming at reasoning about real-time dynamic systems running online in changing environments. 

As future work we can straightforwardly extend these results to probabilistic reasoning with stable models by applying the new module theorem to \cite{DamasioM11}, as well as to DLP functions and general stable models.
An implementation of the framework is also foreseen in order to assess the overhead when compared with the original benchmarks in~\cite{OJtplp08}.
Based on our own preliminary work and results in the literature, we believe that a fully compositional semantics can be attained by resorting to partial interpretations \egc SE-models~\cite{Turner:2003:SEM:986809.986818} for defining program models at the semantic level. It is known that one must include extra information about the support of each atom in the models in order to attain generalized compositionality and SE-models appear to be enough.

\bibliography{NMR2014Modular}
\bibliographystyle{aaai}%aaai,plain

\appendix
\section{Proofs}\label{sec:compsemantics}

\begin{proof}[Lemma~\ref{lem:notcompos}]
A join operation is a function mapping a pair of sets of interpretations into a set of interpretations. Consider the following program modules:
\[
\begin{array}{ll}
\mathcal{P}_1 = < \{a.\}, \emptyset, \{a,b\}, \emptyset > & \mathcal{Q}_1 = < \{a.~~ \bot \leftarrow a,b.\},\\
& \quad\quad\quad \emptyset, \{a,b\}, \emptyset >\\
\mathcal{P}_2 = \left< \{b.\}, \emptyset, \{b\}, \emptyset \right> & \mathcal{Q}_2 = \left< \{b.\}, \emptyset, \{b\}, \emptyset \right>\\
\mathcal{P}_1 \uplus \mathcal{P}_2 = < \{a.~~ b.\}, \emptyset, & \mathcal{Q}_1 \uplus \mathcal{Q}_2 = < \{a.~~ \bot \leftarrow a,b.\\ 
\quad\quad\quad \{a,b\}, \emptyset >&\quad\quad\quad b.\}, \emptyset, \{a,b\}, \emptyset >\\ 
\end{array}
\]
One sees that $AS(\mathcal{P}_1)=AS(\mathcal{Q}_1)=\{\{a\}\}$\@, and $AS(\mathcal{P}_2)=AS(\mathcal{Q}_2)=\{\{b\}\}$ but $AS(\mathcal{P}_1 \uplus \mathcal{P}_2)= \{\{a,b\}\}$ while $AS(\mathcal{Q}_1 \uplus \mathcal{Q}_2)= \{\}$\@. Therefore, it cannot exist $\bowtie'$ since this would require 
$AS(\mathcal{P}_1 \uplus \mathcal{P}_2) = AS(\mathcal{P}_1) \bowtie' AS(\mathcal{P}_2) = \{\{a\}\} \bowtie' \{\{b\}\} = AS(\mathcal{Q}_1) \bowtie' AS(\mathcal{Q}_2) = AS(\mathcal{Q}_1 \uplus \mathcal{Q}_2)$\@, a contradiction. \qed
\end{proof}

\begin{proof}[Theorem \ref{theorem:newmodule}]
By reduction of the conditions of the theorem to the conditions necessary for applying the original Module Theorem.
If $\mathcal{P}_1 \uplus \mathcal{P}_2$ is defined then let their transformed relaxed composition be $T = (\mathcal{P}_1 \uplus^{RT} \mathcal{P}_2)$. It is clear that the output atoms of $T$ are $O_1 \cup O_2$, the input atoms are $(I_1 \cup I_2) \setminus (O_1 \cup O_2)$, and the hidden atoms are $H_1 \cup H_2 \cup O' \cup O''$\@. Note that before the application of the hiding operator the output atoms are $O_1 \cup O_2 \cup O' \cup O''$\@. The original composition operator $\sqcup$ can be applied since the outputs of $\rho_{O' \leftarrow O}(\mathcal{P}_1)$, $\rho_{O'' \leftarrow O}(\mathcal{P}_2)$ and $\mathcal{P}_{union}$ are respectively $O' \cup (O_1\setminus O)$, $O'' \cup (O_2 \setminus O)$ and $O=O_1\cap O_2$, which are pairwise disjoint.
Because of this, we are in the conditions of the original Module Theorem and thus it is applicable to the result of the modified composition $\uplus$ iff the transformation did not introduce positive loops between the program parts of the three auxiliary models. If $\mathcal{P}_1 \uplus \mathcal{P}_2$ had no loops between the common output atoms than its transformation $\mathcal{P}_1 \uplus^{RT} \mathcal{P}_2$ also does not because it results from a renaming into new atoms.

Consider now the rules part of $T$\@; if we ignore the extra introduced atoms in $O'$ and $O''$ the program obtained has exactly the same stable models of the union of program parts of $\mathcal{P}_1$ and $\mathcal{P}_2$\@. Basically, we are substituting the union of $o \leftarrow Body^1_1., \ldots, o \leftarrow Body^1_m.$ in $\mathcal{P}_1$\@, and $o \leftarrow Body^2_1., \ldots, o \leftarrow Body^2_n.$ in $\mathcal{P}_2$ by:
\[
\begin{array}{l@{\qquad}l}
o \leftarrow o'. & o \leftarrow o''.\\
o' \leftarrow Body^1_1. & o'' \leftarrow Body^2_1.\\
\hdots & \hdots\\
 o' \leftarrow Body^1_m. & o'' \leftarrow Body^2_n.\\
 \bot \leftarrow o', \naf o. &  \bot \leftarrow o'', \naf o.\\
\end{array} 
\]
This guarantees visible equivalence of $\mathcal{P}_1 \uplus \mathcal{P}_2$ and  $\mathcal{P}_1 \uplus^{RT} \mathcal{P}_2$\@, since the models of each combined modules are in one-to-one correspondence, and they coincide in the visible atoms. The contribution of the common output atoms is recovered by the joins involving atoms in $O'$, $O''$ and $O$\@, that are all pairwise disjoint, and ensuring that stable models obey to $o = o' \vee o''$ via program module $\mathcal{P}_{union}$\@. The constraints introduced in the transformed models $\rho_{O' \leftarrow O}(\mathcal{P}_1)$ (resp. $\rho_{O'' \leftarrow O}(\mathcal{P}_2))$ simply prune models that have $o$ false and $o'$ (resp. $o''$) true, reducing the number of models necessary to consider. Since the input and output atoms of $\mathcal{P}_1 \uplus \mathcal{P}_2$ and  $\mathcal{P}_1 \uplus^{RT} \mathcal{P}_2$ are the same, then $\mathcal{P}_1 \uplus \mathcal{P}_2 \equiv_m \mathcal{P}_1 \uplus^{RT} \mathcal{P}_2 $\@. 
\qed\end{proof}

\begin{proof}[Theorem~\ref{theorem:conservativemoduletheorem}]The theorem states that if we ignore the renamed literals in $\otimes$ the models are exactly the same, as expected. The transformed program module $\mathcal{P}_1 \utimes \mathcal{P}_2$ corresponds basically to the union of programs, as seen before. Consider a common output atom $o$\@. The constraints in the module part $\mathcal{P}_{filter}$ combined with the rules in $\mathcal{P}_{union}$ restrict the models to the cases for which $o \equiv o' \equiv o''$\@. The equivalence $o \equiv o'$ restricts the stable models of $\rho_{o' \leftarrow o}(\mathcal{P}_1)$ to the original stable models (except for the extra atom $o'$) of $\mathcal{P}_1$, and similarly the equivalence $o \equiv o''$ filters the stable models of $\rho_{o'' \leftarrow o}(\mathcal{P}_2)$ obtaining the original stable models of  $\mathcal{P}_2$\@. Now it is immediate to see that compositionality is retained by making  the original common atoms $o$ compatible\@.
\qed\end{proof}

\end{document}